\documentclass[sn-basic]{sn-jnl}

\usepackage{svg}

\usepackage{tabularx}

\usepackage{kbordermatrix}

\usepackage{graphicx}
\usepackage[algo2e,ruled,vlined]{algorithm2e}

\usepackage{times}
\usepackage{svg}
\usepackage{multirow}
\usepackage{fancyhdr,graphicx,amssymb}
\usepackage{subcaption}
\usepackage{float}
\usepackage{mathtools}
\DeclarePairedDelimiter\abs{\lvert}{\rvert}%

\usepackage{xtab,caption}
\usepackage{systeme}          
\usepackage[graphicx]{realboxes}
\usepackage{pdfpages}

\newcommand{\card}[1]{\ensuremath{\lvert #1\rvert}\xspace}

\newcommand{\svm}{${\tt RankingSVM}$}
\newcommand{\ahp}{${\tt AHPRank}$}
\newcommand{\svmp}{${\tt RankingSVM.0}$}
\newcommand{\ahpp}{${\tt AHPRank.0}$}
\newcommand{\svma}{${\tt RankingSVM.1}$}
\newcommand{\ahpa}{${\tt AHPRank.1}$}
\newcommand{\yules}{${\tt YulesY}$}
\newcommand{\cosine}{${\tt Cosine}$}
\newcommand{\lamb}{${\tt Lambda}$}
\newcommand{\lev}{${\tt Leverage}$}
\newcommand{\cert}{${\tt Certainty}$}
\newcommand{\lap}{${\tt Laplace}$}
\newcommand{\IFa}{${\tt InterestFactor}$}

\newcommand{\nbTrans}{{$\abs{\mathcal{T}}$}}
\newcommand{\nbItems}{{$\abs{\mathcal{I}}$}}
\newcommand{\nbRules}{${\tt \#Rules}$}

\jyear{2022}%

\theoremstyle{thmstyleone}%
\newtheorem{theorem}{Theorem}
\newtheorem{proposition}[theorem]{Proposition}%

\theoremstyle{thmstyletwo}%
\newtheorem{example}{Example}%

\theoremstyle{thmstylethree}%

\begin{document}
\title{Boosting the Learning for Ranking Patterns}


\author[1,2]{\fnm{Nassim } \sur{Belmecheri}}\email{belmecheri.nassim@edu.univ-oran1.dz}

\author[1]{\fnm{Noureddine} \sur{Aribi}}\email{,aribi.noureddine@edu.univ-oran1.dz}
\author[2]{\fnm{Nadjib} \sur{Lazaar}}\email{Nadjib.Lazaar@lirmm.fr}
\author[1]{\fnm{Yahia} \sur{Lebbah}}\email{lebbah.yahia@edu.univ-oran1.dz}
\author[3]{\fnm{Samir} \sur{Loudni}}\email{samir.loudni@imt-atlantique.fr}

\affil[1]{\orgdiv{Lab. LITIO}, \orgname{Université Oran1}, \orgaddress{ \postcode{31000}, \state{Oran}, \country{Algeria}}}

\affil[2]{\orgdiv{LIRMM}, \orgname{Université de Montpellier, CNRS},\orgaddress{ \postcode{34000}, \state{Montpellier}, \country{France}}}

\affil[3]{\orgdiv{ TASC (LS2N-CNRS)}, \orgname{ IMT Atlantique}, \orgaddress{\postcode{44307}, \state{ Nantes}, \country{France}}}


\abstract{
Pattern mining is a valuable tool for exploratory data analysis, but identifying relevant patterns for a specific user is challenging. Various interestingness measures have been developed to evaluate patterns, but they may not efficiently estimate user-specific functions. Learning user-specific functions by ranking patterns has been proposed, but this requires significant time and training samples. In this paper, we present a solution that formulates the problem of learning pattern ranking functions as a multicriteria decision making problem. Our approach uses Analytic Hierarchy Process (AHP) to elicit weights for different interestingness measures based on user preference. We also propose an active learning mode with a sensitivity-based heuristic to minimize user ranking queries while still providing high-quality results. Experiments show that our approach significantly reduces running time and returns precise pattern ranking while being robust to user error, compared to state-of-the-art approaches.
 }




\maketitle

\section{Introduction}

Data mining is a field that focuses on discovering relevant information from data and transforming it into useful knowledge. One of its significant subfields, pattern mining, involves searching and enumerating significant patterns in data. In recent years, there has been a notable shift in the pattern mining community from efficiency-based approaches to methods that can extract more meaningful patterns.
Unfortunately, obtaining relevant results with traditional pattern mining methods can be challenging and time-consuming. There are two main issues: first, a vast number of patterns are discovered, many of which are redundant. Second, the preferences and domain experts' background knowledge is not considered.
The significance of incorporating user preferences and knowledge was initially highlighted by Silberschatz and Tuzhilin \cite{SilberschatzT95}. The fundamental concept is to represent the user's preferences using objective interestingness measures. However, as noted in \cite{Bie11}, using objective quality measures has limited practical utility since interestingness is subject to the particular user and task at hand. The interestingness of a pattern is determined solely based on the data.

In recent years, there has been a growing emphasis on user-centered, interactive, and anytime pattern mining \cite{dzyuba2013interactive,BoleyMKTW13,dzyuba2017learning}. This paradigm highlights the need to quickly present users with patterns that are likely to be interesting to them, and to obtain feedback that will influence subsequent iterations of the interactive mining process. An important aspect of this framework is the ability to learn user-specific functions for ranking patterns from feedback. Although the idea of learning user-specific ranking functions was first investigated in \cite{xin2006discovering}, it has been extended in recent years by Boley et al. \cite{BoleyMKTW13} and Dzyuba et al. \cite{dzyuba2013interactive,dzyuba2017learning} in the context of interactive pattern mining. Additionally, Xuguang Bao et al. \cite{bao2021knowledge} and Liang Chang et al. \cite{chang2022idmbs} have adapted interactive data mining for extracting interesting spatial co-location patterns. These approaches leverage standard machine learning techniques to learn weighted vectors based on selected pattern features, such as items, transactions, and length.

In \cite{dzyuba2013interactive}, a linear ranking function is learned using support vector machines (\svm).
In \cite{BoleyMKTW13,dzyuba2017learning}, the authors propose using Stochastic Coordinate Descent (SCD) \cite{Shalev-ShwartzT11}
to learn a logistic function.
While these methods enable the exploitation of user feedback to learn ranking functions, they also increase the computational cost of the learning task, particularly when the number of pattern-measure pairs used for ranking increases.
Nonetheless, a crucial aspect of this framework is the ability to quickly present a set of patterns to the user and focus on what may be of interest using interestingness measures.

In this paper, we address the challenge of learning pattern ranking functions as a multicriteria decision making problem. To this end, we propose to use a {\it weighted linear function} to aggregate all the individual measures into a global ranking function. Linear functions have been widely adopted in learning to rank models because they are simple, interpretable, and efficient, yet can still effectively capture the underlying relationships between input features and relevance labels \cite{burges2005learning, qin2010general}.

We propose a learning algorithm that is both fast and scalable, which enables us to maintain the expected weights of measures. Our approach leverages the Analytical Hierarchy Process (AHP)~\cite{saaty1988analytic} and a set of user-ranked patterns to learn the weighting vector of measures. Specifically, we seek to maximize the correlation between the unknown user's ranking function and the learned AHP-based ranking function. We previously introduced this approach in our work \cite{RNTI/papers/1002739}, where we presented the initial results of the passive version of the algorithm.

This study presents a new learning algorithm that can operate in both active and passive modes, specifically designed for pattern ranking. To improve performance in active mode, a sensitivity-based heuristic is proposed for selecting patterns. The algorithm is explained in detail using a running example. The approach was validated on large datasets using three different user simulations in both passive and active modes. The results indicate that the proposed approach is robust and can account for user mistakes and changes in ranking criteria in interactive learning scenarios.

We conducted experiments to evaluate the effectiveness of our approach using association rule mining as a case study \cite{10.1145/170036.170072}. Our research questions were as follows: (1) Can user-specific preferences over all patterns be inferred from a sample of ranked patterns, and if so, what is the required input data size and learning time? (2) How does our new learning method based on AHP compare to the currently used method for learning user-specific ranking functions? (3) Is the sensitivity heuristic an appropriate choice for query selection in active learning mode? (4) How does the robustness of our approach compare to others in the face of user errors?
Our experiments on various datasets demonstrated that our approach significantly reduces running time while accurately ranking patterns compared to state-of-the-art methods.

The paper is organized as follows: 
In Section 2, we review the related work in pattern mining and ranking functions. 
In Section 3, we provide some necessary preliminaries for our approach. 
Section 4 formulates the problem we tackle and describes our proposed method. 
A running example is presented in Section 5 to illustrate our approach. 
In Section 6, we report and discuss the empirical results of our experiments. 
Finally, we conclude the paper in Section 7.

\section{Related work}
Iterative/user-centric pattern mining has gained significant attention in recent years, where the user is involved in the mining process to improve the mining results by providing feedback. This approach is particularly useful in domains like medical diagnosis or fraud detection, where the user's domain knowledge is essential. To make this process more efficient and effective, researchers have been exploring methods to learn ranking functions that can quickly present patterns to the user based on their interests. This problem, known as object ranking, has been studied extensively in machine learning and data mining communities.

One common approach to object ranking is to learn a ranking function based on pairwise comparisons. However, these methods can be time-consuming and expensive since they require a large number of pairwise comparisons to learn a reliable ranking function. To address this challenge, researchers have explored alternative methods such as multicriteria decision making. In our approach, we use the Analytical Hierarchy Process (AHP) to learn the measures' weighting vector, which maximizes the correlation between the user's ranking function and the learned AHP-based ranking function.
In \cite{xin2006discovering}, the authors proposed a user-centric approach to mining interesting patterns based on the user's feedback. They developed a log-linear model for itemsets that considers the user's prior knowledge about the items and their relationships. The model was trained using \svm{} formulations to learn a ranking function that assigns a score to each itemset. The authors also proposed a belief model that exploits belief probabilities assigned to transactions for more complex patterns. The authors demonstrated that their approach outperformed existing methods in terms of accuracy and efficiency for complex patterns where traditional methods may fail to capture the user's preferences.

Dzyuba et al.\cite{vlad} extended the approach proposed by Xin et al.\cite{xin2006discovering} and used \svm{} for active preference learning to rank patterns. However, these methods suffer from computational inefficiencies, which limit their scalability on large datasets. In contrast, the AHP-based approach proposed in our work has a linear time complexity and can efficiently rank a large number of patterns.

In \cite{bhuiyan2016priime}, the authors proposed a user-centric generic framework (PRIIME) for learning ranking functions. In order to achieve this, regression techniques based on neural networks were adopted. . They assigned a score to each pattern based on user feedback and trained a regression model on labeled patterns and their corresponding scores. The model predicts the score of new patterns based on their features and the user feedback. The proposed method outperformed existing methods in terms of accuracy and efficiency and can effectively mine interesting patterns and provide valuable insights to the user.

In \cite{bao2021knowledge}, an efficient interactive approach is proposed for users to discover their preferred co-location patterns. The approach consists of a flexible interactive framework, an ontology-based similarity measure between co-location patterns, a pattern filtering model to express user preferences, and a pruning scheme to reduce the number of outputs.

In \cite{chang2022idmbs}, the authors propose IDMBS (Interactive Data Mining based on Support Vector Machine), an interactive mining system that uses SVM to discover user-preferred co-location patterns. The system includes a filtering algorithm to select patterns, which are then annotated by the user, and an SVM model to train on these patterns and discover additional user-preferred co-location patterns.

Our proposed approach tackles the problem of learning a user-specific function for ranking patterns as a multicriteria decision making problem using AHP. Moreover, our method remains fast when the number of patterns used for ranking increases.

Most of the related work uses SVM-based algorithms to learn a user-specific function. According to \cite{bottou2007support}, the worst-case time complexity of the SVM algorithm varies between $k^2$ and $k^3$ on the number of samples $k$; whereas our approach is linear on $k$  (see section \ref{sec:complexity}).

\newcommand{\attrs}{\mathcal{A}}
\newcommand{\objs}{\mathcal{O}}
\newcommand{\rels}{\mathcal{L}}
\newcommand{\patts}{\mathcal{L}}
\newcommand{\ranks}{\mathcal{K}}
\def\N{\mbox{$\mathbb{N}$}}
\newcommand{\setRp}{{\mathord{{\mathbb R_{+}}}}}

\section{Background}
In this section, we provide the necessary background information about pattern mining, interestingness measures commonly used to evaluate pattern quality, and multi-criteria decision making using the AHP method.
\subsection{Pattern mining and interestingness measures}
 Data of all types can be represented as a triplet ($\objs,\attrs,\rels$), where $\objs$ is a set of objects, $\attrs$ is a set of attributes, and $\rels$ is the relational language used to express objects with respect to a set of attributes. This framework is inspired by the concepts of inductive databases \cite{DBLP:journals/sigkdd/Raedt02,DBLP:journals/cacm/ImielinskiM96}.
The language $\rels$ can be used to define various types of datasets. For instance, when $\rels$ is a binary relation between attributes and objects (i.e., $\rels\subseteq\objs\times\attrs$), we get the classical itemset mining framework. When $\rels$ expresses each object as a sequence of attributes, we get a sequential dataset. Finally, when $\rels$ expresses each object as a graph where nodes and edges are labeled, we get a dataset of graphs.

The goal of data mining is to uncover meaningful patterns or regularities between objects within large datasets. To evaluate the quality of these patterns and determine the most relevant ones, interestingness measures are employed.

The field of data mining is actively researching how to measure the interestingness of discovered patterns with respect to user-specific preferences. Nine concepts have been introduced to assess the interestingness of patterns: \emph{Conciseness, Generality, Reliability, Peculiarity, Diversity, Novelty, Surprisingness, Utility}, and \emph{Applicability} \cite{GengH06}. Some of these concepts are correlated, such as \emph{Surprisingness} and \emph{Applicability}, while others are conflicting, like \emph{Generality} and \emph{Pecularity}, and some are independent, like \emph{Novelty} and \emph{Utility}. Depending on the user's preferences, they may value certain concepts over others. To satisfy these preferences, a wide range of measures and aggregation techniques have been proposed \cite{geng2006interestingness,kuznetsov2018interestingness}.

\subsection{ Analytical Hierarchy Process (AHP)} 
\label{sec:ahp:background}
AHP is a popular multicriteria decision-making method created by Saaty~\cite{saaty1988analytic}. It involves breaking down the decision problem into a hierarchical structure and deriving criteria weights from the decision maker using pairwise comparisons, rather than relying solely on numerical values.
After constructing a decision problem hierarchy using the AHP method, the importance weights of criteria can be determined through the following three main steps:
\begin{enumerate}
	\item To estimate the relative importance of $m$ criteria at a given level, the user is asked to make pairwise comparisons between them. These comparisons are used to construct a pairwise comparison matrix {\bf A} of size $m \times m$, where $a_{ij}$ represents the importance of criterion $i$ relative to criterion $j$. If $a_{ij} > 1$, it means that criterion $i$ is considered more important than criterion $j$.
	To ensure consistency, the matrix is filled such that $a_{ij} = 1/a_{ji}$ and $a_{ii} = 1$. AHP uses a scale of 1 to 9 to quantify the preference degree, where 1 indicates indifference and 9 indicates absolute preference.
	
	
	\item  After constructing the pairwise comparison matrix at a given level, various mathematical techniques proposed in the literature can be used to compute the weight vector $w = (w_1, \ldots, w_m)^T$. These techniques include the eigenvector method (EVM)~\cite{SAATY1984-comp,TAKEDA1987-evm} and distance-based minimization methods, such as the Least Squares Method, Logarithmic Least Squares Method, Weighted Least Squares Method, Logarithmic Least Absolute Values Method, and Singular Value Decomposition~\cite{GASS2004-SVD,Blankmeyer1987ApproachesTC, brunelli2014introduction}. The EVM method is the most commonly used and computes the weight vector $w$ by solving the following characteristic equation:
	\begin{equation}
		\label{eq:ahp:w}
		\left\{
		\begin{array}{l}
			{\bf A}\cdot w = \lambda_{max} \cdot w \\ 
			w^T {\bf 1} = 1
		\end{array}	\right.
	\end{equation} 
	where {\bf A} is the pairwise comparison matrix, $\lambda_{max}$ is the highest eigenvalue of {\bf A}, and 
	${\bf 1} = (1,\dots,1)^T$.
	The constraint $\sum_{i=1}^{m} w_i = 1$ is added to help the solving process and avoid infinitely many solutions. 
	Note that, if {\bf A} is positive then the highest eigenvalue is real (i.e. $\lambda_{max} \in \mathbb{R}$)~\cite{brunelli2014introduction}.
	Interestingly, Saaty~\cite{SAATY1977} mentioned that the perfect\footnote{When the components of {\bf A} are exactly obtained as the ratio between weights.} normalized eigenvector $w$ satisfies the following relation:
	
	\begin{equation}
		\label{eq:ahp:relation}
		a_{ij} = \frac{w_i}{w_j},\, \forall i, j = 1 \ldots m.
	\end{equation} 
	
	In this case, $\lambda_{max}=m$ only when the comparison matrix is perfectly consistent; otherwise, $\lambda_{max}$ is greater than $m$. This equation explains why some researchers prefer distance-based methods, which aim to minimize the distance between $(a_{ij})_{m\times m}$ and $(w{i}/w_{j})_{m\times m}$ directly. However, this leads to a non-linear and difficult-to-solve problem. Therefore, in this paper, we use the EVM method instead of distance-based methods.
 \item Since AHP does not guarantee the transitivity of the user's preferences, the pairwise comparison may result in inconsistencies. To detect these inconsistencies, a consistency check is used by calculating the consistency ratio (CR), defined as $CR = \frac{CI}{RI}$, where $RI$ is a constant obtained from the Random Consistency Index table of AHP. The consistency index of the EVM method is computed as follows:
	\begin{equation}
		\label{eqn:AHP}
		CI = \frac{(\lambda_{max} - n)}{n - 1}
	\end{equation}
	The reliability of the computed weighting vector $w$ (solution of Equation~(\ref{eq:ahp:w})) is assessed by checking whether the CR value is less than $0.1$.
\end{enumerate}

It is common for hierarchies to contain multiple levels of criteria, especially when there are more than nine criteria. To handle this, it is beneficial to establish a hierarchy based on existing dependencies between the criteria. This can help optimize the decision-making process and reduce the number of required pairwise comparisons. For further information on this topic, we recommend referring to~\cite{brunelli2014introduction}. However, in this paper, we will not delve into this aspect. In Example \ref{exple:ahp}, we illustrate the AHP process.
\begin{example}
\label{exple:ahp}
Consider a decision problem involving three criteria, where we aim to determine their relative importance in achieving an overall goal. To do so, we construct a pairwise comparison matrix {\bf A}$(a_{ij})_{3\times 3}$, where each element $a_{ij}$ represents the relative importance of criterion $i$ compared to criterion $j$. The matrix is given as follows:
$$ {\bf A} =
\begin{pmatrix}
        &M_1&M_2&M_3\\
	M_1&1 & 1/2 & 1/4\\
	M_2&2 & 1 & 1/2 \\
	M_3&4 & 2 & 1
\end{pmatrix}
$$

To determine the weights of the criteria, we can employ the EVM method, which involves solving the equation ${\bf A} w = \lambda_{max} w$. The resulting weight vector is given by $w = (w_1, w_2, w_3) = (1/7, 2/7, 4/7)$, where each component represents the weight of the corresponding criterion.   
\end{example}

\newcommand{\kendall}{{\tt kendallW}}
\section{AHP-based learning approach for ranking patterns}
\label{sec-AHP-based-learning}

\newcommand{\learnWeights}{{\tt learnWeights}}

\newcommand{\learning}{{\tt computeCorrelation}}

\newcommand{\buildMatrix}{{\tt buildAHPMatrix}}

\newcommand{\SBH}{{\tt BaseHeuristic}}

\newcommand{\weight}{{\tt computeWeights}}
	\SetKw{return}{return}

We aim to develop a ranking function that combines a set of interestingness measures while also taking into account user-specific preferences.
	
\subsection{Problem statement: learning pattern rankings}

In this section, we demonstrate that learning a ranking function that incorporates a set of interestingness measures to approximate user-specific preferences can be formulated as a multicriteria decision-making problem. In this context, the criteria correspond to the interestingness measures.
%
The objective of this task is to learn a ranking function for patterns based on a sample of ordered patterns. To accomplish this, we adopt the approach of \textit{ordered feedback}, where users are asked to provide a total order over a small set of patterns based on their subjective interestingness. This approach is inspired by previous research~\cite{DzyubaLNR14}.

Let $\triangleright$ be a binary preference relation between patterns. A ranking function based on $\triangleright$, denoted by $r_{\triangleright}(.)$, takes as input a subset of patterns ${\cal P} \subseteq \rels$ and returns as output a permutation $S$ $=$ $\langle P_{r_1}, P_{r_2}, \ldots,P_{r_n}\rangle$ of $\mathcal{P}$, such that $\card{\mathcal{P}}=\card{S}$ and $\forall i,j \in [1, n], i < j, $ $P_{r_i} \triangleright P_{r_j}$.

We aim to learn the user ranking function denoted by $r_{\triangleright_u}$, which is based on the user preference relation $\triangleright_{u}$. To achieve this, we use a set of interestingness measures $\mathcal{M}=\{M_1, \ldots M_m\}$ and their corresponding ranking functions $r_{\triangleright_{M_1}}, \ldots, r_{\triangleright_{M_m}}$. The rank of pattern $P_j\in \mathcal{P}$ w.r.t. the interestingness measure $M_i$ is denoted by $rank_{M_i}(P_j)$. Here, $(P_k \triangleright_{M_i} P_l)$ means that $M_i(P_k)\geq M_i(P_l)$.

To frame the ranking problem as a learning problem, we need to find the weights $w_i$ of a linear aggregation function over the measures from a set of ranked patterns that maximizes the correlation with the user ranking function $r_{\triangleright_u}$. Specifically, we aim to find the coefficients of the following function by solving a multi-criteria optimization problem:

\begin{equation}\label{eq:aggregation}
	\resizebox{0.5\hsize}{!}{$
		f_{w}(P)=\sum_{i=1..m} w_i \,.\,M_i(P)$}•
\end{equation}\\


Our approach leverages user feedback to perform pairwise comparisons between measures and estimate their relative importance based on individual measure scores. The pairwise comparison results are then aggregated into a matrix, which is used by the Analytic Hierarchy Process (AHP) to compute the weights $w_i$ for each measure.

\paragraph{\bf Comparing a single  measure to a user ranking. }
Let $S= \langle P_{1}, \ldots, P_{n}\rangle$ be a user ranking of patterns.
Let $\mathcal{M}=\{M_1, \ldots, M_m \}$ be a set of interestingness measures for patterns.
To evaluate how well a particular measure performs in terms of overall ranking accuracy compared to a user ranking, we use Kendall's W concordance coefficient \cite{kendall1939}.
For a measure $M_i$ and a user ranking $S$, we define $K_i(S)\in [0,1]$ as the Kendall's W coefficient between the ranking obtained using $M_i$ and the user ranking $S$, where a value of $1$ indicates perfect concordance and a value of $0$ indicates no concordance.



\begin{equation}\label{equ-kendall-w-formula}
\resizebox{0.5\hsize}{!}{$
K_i(S) = \frac{12\alpha}{{\abs{S}}^{3}-{\abs{S}}} \ ,\quad 
\alpha=\sum _{\ell=1}^{\abs{S}} (R_{\ell}-{\bar {R}})^{2}, $}
\end{equation}

\noindent

where $R_{\ell} = rank_{M_i}(P_\ell) + rank_{u}(P_\ell)$ with $P_\ell\in S$, 
and $\bar{R} = \frac {1}{\abs{S}}\sum\limits_{\ell=1}^{\abs{S}} R_{\ell}$.


\paragraph{\bf Comparing two  measures according to a user ranking.}
Let $M_i$ and $M_j$ be two measures to be compared based on a given user ranking $S$. We can compute the gap $\Delta_{i,j}(S)= K_{i}(S)-K_{j}(S)$ to determine how much $M_i$ is closer to the user ranking $S$ than $M_j$. If $\Delta_{i,j}(S)\geq 0$, then $M_i$ is closer, otherwise, the converse is true.
When considering a set of user rankings $\mathcal{S}=\{S_1,\ldots,S_n\}$, comparing two measures $M_i$ and $M_j$ on $\mathcal{S}$ involves estimating the measure closest to $\mathcal{S}$ by averaging the gaps over all the user rankings in $\mathcal{S}$: $\Delta_{i,j}=\frac{1}{n} \sum\limits_{S_k\in\mathcal{S}} \Delta_{i,j}(S_k)$.

\subsection{Learning weights process}

The function \ref{alg:AHPBoost} implements the learning process that computes a weight vector $w$(aweight for each measure) by exploiting the pairwise comparisons of measures.
It takes as inputs a vector of pairwise comparisons $\Delta$ and its corresponding index $l$, a user ranking $S$, and a set of measures $\mathcal{M}$.
The learning step is performed at line \ref{f:learn} where for each pair of measures, we first compute the Kendall W coefficient for each measure $i$ with respect to the user ranking $S$ ($K_i(S)$). Then, we compute the gap between the pairs.

The function \ref{alg:AHPBoost} implements the learning process that computes a weight vector $w$~(a weight for each measure) by exploiting the pairwise comparisons of measures. It takes as inputs a vector of pairs comparisons $\Delta$ and the corresponding index $l$, a user ranking $S$, and a set of measures $\mathcal{M}$. The function incrementally updates $\Delta$ by taking into account the previous $l$ user rankings at line \ref{learn:kendall}. At line \ref{f:build}, the AHP matrix is constructed by scaling the $\Delta_{i,j}$ comparisons to the AHP values between ${1}/{9}$ and 9. Once constructed, the criteria weight vector $w$ is returned in line \ref{ahp:learnWeights} by solving the minimization problem described in line \ref{f:weight}, using the eigenvector method (EVM) \cite{brunelli2014introduction}.


\SetKwInOut{InOutput}{In Out}
\SetKwInOut{Input}{In}
\SetKwInOut{Output}{Out}

\begin{algorithm}[H]

	\LinesNumbered

	\SetAlgoLined
 \floatname{algorithm}{Algorithm}

	\Input{Set of measures $\mathcal{M}=\{M_1,\ldots, M_m\}$; 
		user ranking $S= \langle P_{1}, \ldots, P_{n}\rangle$;}
	
	\Output{Weight vector $w$;}

	
	
		
	\BlankLine
	\BlankLine
			${\bf A}[i,j] \gets 1,\,   \forall i,j \in \{1,\ldots,m\}$; 

	\BlankLine
	\BlankLine
		
			\ForEach{$M_i,M_j\in \mathcal{M}: i<j$}{\label{f:learn}
				$\Delta_{i,j}\gets (K_{i}(S)- K_{j}(S))$;\ \label{learn:kendall}
			}
			
				\BlankLine
				\BlankLine

		\ForEach{$\Delta_{i,j}\in \Delta$}{\label{f:build}
			\lIf{ $\Delta_{i,j} < 0 $}{
				
				\quad scale $\Delta_{i,j}$ to $(-9..-1);\quad$
				${\bf A} [j,i] \gets \abs{\Delta_{i,j}};\quad$
				${\bf A} [i,j] \gets 1/\abs{\Delta_{i,j}}$
			}
			\lElse{
				
				\quad scale $\Delta_{i,j}$ to $(1..9);\quad$ 
				${\bf A}[i,j] \gets \Delta_{i,j}; \quad$
				${\bf A}[j,i] \gets 1/\Delta_{i,j}$					
			}
		}	
		$w \gets solve\left(
		\begin{array}{l}
			{\bf A}\cdot w = \lambda_{max} \cdot w \\
			\text { subject to } \sum_{j=1}^{m} w_{j}=1, \quad w_{j}>0, \forall j=1 \ldots m .
\end{array}\right)$ \label{f:weight}
		
\Return{$w$;}\label{ahp:learnWeights}

	\caption{$\learnWeights(\langle\Delta, l\rangle, S, \mathcal{M})$\label{ahp:returnlearnweights}}\label{alg:AHPBoost}
	
\end{algorithm}

After learning the weight vector $w$, we use it to compute the score of a given pattern $P_i$ by applying a weighted aggregation function:

\begin{equation}
	g_{w}(P_i)=\sum\limits_{M_i\in \mathcal{M}}w_i\ scale({M_i}(P)) 
	\label{eq:score}
\end{equation}  

The $scale$ function is used to normalize the interestingness values of the different measures to a $[0,1]$ scale so that they can be combined using the weighted aggregation function.


\subsection{AHPRank algorithm}
\label{sec-AHPRank-algorithm}
We propose Algorithm \ref{alg:ActiveLearning} to implement the \ahp{} function (Algorithm \ref{alg:AHPBoost}) for learning a weight vector in a passive/active mode. The algorithm takes as input a set of measures $\mathcal{M}$, a set of user-ranked patterns $\mathcal{S}$, and a triplet of parameters $\langle \mathcal{P},\theta,T\rangle$. Here, $\mathcal{P}$ is a collection of patterns, $\theta$ is a sample parameter used in the query generator for active learning, and $T$ is the maximum number of iterations used as a stopping criterion in the algorithm.

\ahp{} initializes the vector of pairs comparisons to zero in line \ref{ahp:init}.
If \ahp{} is called with a non-empty $\mathcal{S}$, it runs in passive mode by calling $\learnWeights$ function on the given user-ranked patterns $\mathcal{S}$ (lines \ref{ahp:startpassive}-\ref{ahp:endpassive}). This produces a weight vector $w$.
Otherwise, it runs in active mode by submitting $T$ queries to the user (lines \ref{ahp:startactive}-\ref{ahp:endactive}). In the active mode, the user is asked to rank a subset of patterns proposed by a query generator.
Our query generator is based on a heuristic that exploits the quality of the weights learned from the previous iterations. The active learning mode follows the following iterative steps:
\begin{enumerate}

\item As a first step, a subset of patterns is sampled from the global set of patterns $\mathcal{P}$ using a random sampling approach. At line \ref{ahp:sample}, the function ${\tt SamplePatterns}$ is called, which returns a sample $\mathcal{P'}$ of $\theta$ patterns.
\item  Use the ${\tt SensitivityBasedHeuristic}$ function to select a pair of patterns. This heuristic is designed to select a high-quality pair of patterns to accelerate the convergence of the learning process (see function \ref{alg:sbh} in section \ref{sec-heuristic} for details).
\item Interacting with the user to obtain feedback through a ranking query (line \ref{ahp:feedback}).
\item AHP-based learning of the user ranking function is performed by calling the $\learnWeights$ function on the user-provided ranking query (line \ref{ahp:dolearning}).
\end{enumerate}

The proposed method is designed to handle large sets of patterns, which can be computationally expensive to process. To address this issue, we introduce two steps in the active learning mode. First, we randomly sample a subset of patterns from the global set $\mathcal{P}$, which reduces the complexity of the second step as well as introduces a diversity factor. Second, we use a sensitivity-based heuristic to select pairs of patterns that are of high quality and diverse.

It is important to determine when to stop the learning process. Our method allows for a custom stopping criterion, but we also consider the case where a fixed number of iterations $T$ is used. This simulates the scenario where the user can stop at any time based on their satisfaction with the learned ranking function.

\begin{algorithm}[h]
	\LinesNumbered
	\SetAlgoLined
	 \setcounter{AlgoLine}{0}
	\Input{Set of measures $\mathcal{M}=\{M_1,\ldots, M_m\}$;\\
	Passive mode ($\mathcal{S} \neq \emptyset$): Set of user-ranked patterns $\mathcal{S}=\{S_1,\ldots, S_n\}$; \\
	Active mode ($\mathcal{S} = \emptyset$): Collection of patterns $\mathcal{P}$; Sample size $\theta$ ;\\ Number of iterations $T$; }
	
	\KwOut{ $ w^t:$ weighting vector at the last iteration $t$;}
	
	\BlankLine

%

\lForEach{$\Delta_{i,j}\in\Delta$}{$\Delta_{i,j}\gets 0$}\label{ahp:init}

	\If{$\mathcal{S} \neq \emptyset$}{

\ForEach{$S_k\in \mathcal{S}$}{ \label{ahp:startpassive}

$w^t \gets {\tt LearnWeights}(\langle\Delta, k\rangle,S_k, \mathcal{M})$ \tcp*{Learn weights} \label{ahp:dopassivelearning}

}

	} \label{ahp:endpassive}
	\Else { \label{ahp:startactive}
		${w}^0 \gets \mathbf{1}$ \tcp*{Weight vector }  
		
\For{$t = 1, 2 \ldots T$}{	
$\mathcal{P'}\gets {\tt SamplePatterns}(\mathcal{P},\theta)$ \tcp*{Step 1} \label{ahp:sample}

$Q_t \gets {\tt SensitivityBasedHeuristic}(\mathcal{P'}, {w}^{t-1}, \mathcal{M})$ \tcp*{Step 2}\label{ahp:heuristic}

$S^*_{t} \gets {\tt AskRanking}(Q_{t})$ \tcp*{Step 3} \label{ahp:feedback}
$w^t \gets {\tt LearnWeights}(\langle\Delta, t\rangle,S^*_t, \mathcal{M})$ \tcp*{Step4} \label{ahp:dolearning}
}
} \label{ahp:endactive}
\return{$w^t$}

\caption{{${\tt  AHPRank}(\mathcal{M}, \mathcal{S}, \langle\mathcal{P}, \theta, T\rangle)$}}\label{alg:ActiveLearning}
	
\end{algorithm}

\begin{algorithm}
	
	\setcounter{AlgoLine}{0}
	\LinesNumbered
	
	\SetAlgoLined
	\floatname{algorithm}{Function}
	
	\Input{Collection of patterns $\mathcal{P}$; Weight vector $w$; Set of measures $\mathcal{M}$; 
	}
		\KwOut{ $ Pair:$ Selected pair of pattern}

	\BlankLine

$\mathcal{P'} \gets {\tt Sort}(\mathcal{P},w)$ \tcp*{sort $\mathcal{P}$ according to $g_w$}

$Pair \gets \emptyset$;

$\sigma_{min} \gets +\infty$;

\For{$i \in 1..\card{\mathcal{P'}}: P_i,P_{i+1}\in \mathcal{P'}$}{
	
	$score \gets \sigma(P_i,P_{i+1})$;\\
	\If{$score < \sigma_{min}$}{
		$\sigma_{min} \gets score$;\\
		$Pair\gets \{P_i,P_{i+1}\}$;\\
	}	
	
}	

\Return{$Pair$}; 	 \label{sbh:end}
	
	\caption{${\tt SensitivityBasedHeuristic}(\mathcal{P}, w, \mathcal{M})$}\label{alg:sbh}
	
\end{algorithm}

\subsubsection{A Pattern Selection Heuristic for AHPRank}
\label{sec-heuristic}
Active learning is a challenging task since each learning step needs to enhance the learned function. The selection of appropriate patterns for the user to rank is the key to improving the learning process incrementally. However, generating the ideal set of patterns is a computationally demanding task and is known to be NP-hard \cite{DBLP:journals/jmlr/Ailon12}. To address this issue, we propose a heuristic (Function~\ref{alg:sbh}) based on sensitivity analysis for AHP models~\cite{10.1093/imaman/3.1.61}. Our heuristic aims to select a pair of patterns that are close in terms of their overall score predicted by the learned function relative to the sum of the measure gaps:

\begin{equation}
     \sigma(P_i,P_j)=\left\lvert \frac{g_w(P_i) - g_w(P_j)}{\sum_{\ell=1}^{m}( M_{\ell}(P_i) - M_{\ell}(P_j) )}\right\rvert ; i,j \in 1 \ldots k.
 \end{equation}\\
 
A successful aggregation should have the property that if two patterns are close in overall predicted score, they should also be close in their measure values. Thus, the algorithm will be enforced into learning the correct preferences of the patterns when having the lowest $\sigma$ value since the selected pair is surely close in terms of $g_w$ values and very distant on the values of the measures. Example \ref{ex:sbh} illustrates the selection process of the function \ref{alg:sbh}.

\begin{example}
\label{ex:sbh}
Suppose we have two patterns $P_1$ and $P_2$ with four interestingness measures $\mathcal{M}={M_1,M_2,M_3,M_4}$. Let $P_1 = \langle 0.1,0.2,0.3,0.7 \rangle$, $P_2 = \langle 0.7,0.3,0.2,0.1 \rangle$, and let the current learned weight vector be $w=\langle 0.25,0.25,0.25,0.25 \rangle$. At this stage, $g_w(P_1)=0.26$ and $g_w(P_2)=0.26$. Although $P_1$ and $P_2$ have the same overall score, their measure values differ. This pair of patterns is an interesting candidate with $\sigma(P_1,P_2)=0$, where getting the user's preference between them can improve the learning process and bring it closer to the user's ranking.

\end{example}

\subsection{Complexity Analysis} \label{sec:complexity}
\begin{proposition}[\bf Time complexity] 
Algorithm~\ref{alg:ActiveLearning} has a time complexity of $O(n\ k)$ in passive learning and $O(n)$ in active learning when the number of interestingness measures is bounded.
\end{proposition}

\begin{proof}
Let $\mathcal{S}=\{S_1,\ldots,S_n\}$ be a set of $n$ user rankings, where $k$ is the size of the largest ranking $S_t$ and $m$ is a bounded number of measures in $\mathcal{M}$. Once the pairwise comparison matrix ${\bf A}$ is built, the preference vector of weights $w$ can be computed using various mathematical techniques, including the eigenvector-based method (EVM) \cite{brunelli2014introduction}. The worst-case time complexity of the EVM approach is $O(m^3)$ \cite{Eigenproblem99}, which is of constant time complexity $O(1)$ since $m$ is constant.
The time complexity of computing the matrix at line~\ref{learn:kendall} of the $\learnWeights$ function is $O(m\ n\ k)$ since Kendall's W is in $O(k)$ (see Definition (\ref{equ-kendall-w-formula})). The complexity of lines $\ref{f:learn}$, $\ref{f:build}$, and $\ref{f:weight}$ in $\learnWeights$ are respectively $O(n\ m^2 + m\ n\ k)$, $O(m^2)$, and $O(m^3)$.
Notice that in AHP, it is demonstrated in \cite{saaty2003magic} that the number of criteria is recommended to be no more than seven $\pm 2$. Thus, assuming $m$ is a small and fixed number of measures, we can set $m\leq 9$ and obtain an asymptotic quadratic complexity of $O(n\ k)$ when AHP is in passive mode. In other words, the approach takes $O(k)$ for each ranking in the given $\mathcal{S}$.
When AHP is in active mode, where queries are pairs of patterns ($k=2$), the time complexity is linear on the number of queries submitted to the user, i.e., $O(n)$.
\end{proof}

It is important to emphasize that in practice, the set $\mathcal{S}$ is typically reduced to a single large ranking in passive learning mode ($n=1$ and a large $k$), and to a set of $n$ queries consisting of pairs of patterns in active learning mode (large $n$ and $k=2$). As a result, \ahp{} can run in linear time, with a complexity that is linear in $k$ for passive mode and linear in $n$ for active mode. This low complexity makes \ahp{} a fast approach for learning a user ranking function, as supported by the experimental evaluation in Section~\ref{sec:xp}.

\section{Running Example}
\newcolumntype{Y}{>{\centering\arraybackslash}X}

To demonstrate our approach, we utilize an initial set of patterns $\mathcal{P}=\{P_1,\ldots,P_{10}\}$, along with a user ranking $r_{\triangleright_u}$, and five interestingness measures $\mathcal{M}=\{M_1,\ldots, M_5\}$. The user ranking for $\mathcal{P}$, rankings given by $M_i$ on $\mathcal{P}$, and the \ahp{} results are presented in the columns of Table \ref{tab:patterns}. It is important to note that none of the measures perfectly match the user ranking (column $user$).

\begin{table}[h]
	\centering
	\begin{tabular} {|c|c||cc|cc|cc|cc|cc||cc|} 
		\hline
		$user$ & $S$ &    \multicolumn{2}{c|}{$M_1$ -- $\mathit{rank}$} &    \multicolumn{2}{c|}{$M_2$ -- $\mathit{rank}$} &    \multicolumn{2}{c|}{$M_3$ -- $\mathit{rank}$} &    \multicolumn{2}{c|}{$M_4$ -- $\mathit{rank}$}&   \multicolumn{2}{c||}{$M_5$ -- $\mathit{rank}$} &  \multicolumn{2}{c|}{$g_w$ -- $\mathit{rank}$} \\ 
		\hline
		 \textbf{1}    & $P_7$        &0.95 & \textbf{1}  & 0.48 & 8  & 0.79 & 2  & 0.30 & 9  & 0.80 & \textbf{1} &0.72  &\textbf{1} \\ 

		 \textbf{2}   &$P_3$       & 0.75 & 3  & 0.72 & 1  & 0.78 & 3  & 0.70 & \textbf{2}  & 0.61 & \textbf{2} & 0.68 &\textbf{2}\\ 

		 \textbf{3}   &$P_6$        & 0.80 & 2  & 0.49 & 7  & 0.50 & 9  & 0.65 & 4  & 0.60 & \textbf{3}  & 0.61 &\textbf{3}\\ 
		
		 \textbf{4}    & $P_1$       & 0.47 & 9  & 0.47 & 9  & 0.76 & 5  & 0.56 & 6  & 0.59 & \textbf{4}  & 0.54 & \textbf{4}\\ 
		
		 \textbf{5}    & $P_8$       & 0.56 & 6  & 0.65 & 4  & 0.63 & 8  & 0.69 & 3  & 0.40 & \textbf{5}  & 0.53 &\textbf{5} \\ 

		 \textbf{6}    & $P_{10}$       & 0.57 & 5  & 0.50 & \textbf{6}  & 0.80 & 1  & 0.4  & 8  & 0.02 & 10  & 0.34&8 \\

		 \textbf{7}    & $P_5$        & 0.62 & 4  & 0.62 & 5  & 0.66 & 6  & 0.57 & 5  & 0.27 & \textbf{7}  &0.48 &\textbf{7} \\ 
		
		 \textbf{8}    & $P_2$        & 0.48 & \textbf{8}  & 0.66 & 3  & 0.65 & 7  & 0.1  & 10 & 0.05 & 9  &0.33 &9 \\ 		
		
		 \textbf{9}    & $P_4$       & 0.50 & 7  & 0.68 & 2  & 0.77 & 4  & 0.50 & 7  & 0.35 & 6  & 0.50&6 \\ 
				
		 \textbf{10}    & $P_9$       & 0.02 & \textbf{10} & 0.1  & \textbf{10} & 0.05 & \textbf{10} & 0.8  & 1  & 0.25 & 8  &0.18 &\textbf{10} \\ 
		
		\hline
	\end{tabular}
	\caption{Running example with $10$ patterns and $5$ measures.\label{tab:patterns}} 
\end{table}

The data presented in Table \ref{tab:patterns} is utilized by our algorithm to construct the AHP Matrix and learn a weight vector $w$ over the set of interestingness measures $\mathcal{M}=\{M_1,\ldots, M_5\}$. For simplicity, we refer to the AHP algorithm in the passive mode with $\mathcal{S}=\{S_a=\langle P_3,P_1,P_5,P_2,P_4\rangle \}$. In this case, the user's ranking order is $P_3 \triangleright_u P_1 \triangleright_u P_5 \triangleright_u P_2 \triangleright_u P_4$. AHP learns the weights by invoking $\learnWeights$ on $S_a$ and computing the correlation between the user rankings of $S_a$ and the rankings of the measures $M_i$. The matrix in (\ref{eqn:delta-mat}) shows the $\Delta$ values for each measure pair. It is worth noting that no two measures perfectly match the user ranking (column $user$) for the pattern set $\mathcal{P}$.

To compute $\Delta_{1,2}(S_a)$, we need to calculate Kendall's W $K_{1}(S_a)$ (resp., $K_{2}(S_a)$) between the ranking of $M_1$ (resp., $M_2$) and the user ranking $S_a$. The value of $\Delta_{1,2}(S_a)$ is given by $(K_{1}(S_a)-K_{2}(S_a))$, which in this case is $(0.25 - 0.25) = 0$. After scaling the values of $\Delta$ to $(-9..-1)$ for negatives and $(1..9)$ for positives, we can observe from Equation~(\ref{eqn:delta-mat}) that measure $M_5$ is better than $M_4$ with a degree of $6$. The same degree of preference is observed between $M_5$ and $M_3$, while $M_3$ and $M_4$ are indifferent.

\renewcommand{\kbldelim}{(}
\renewcommand{\kbrdelim}{)}
{\footnotesize
	\begin{equation}
		\label{eqn:delta-mat}
		\Delta = \kbordermatrix{
			& (M1) & (M2) & (M3) & (M4) & (M5) \\
			(M1) &  & 0 &	0.42 &	0.42 &	-0.43 \\
			(M2)&& &	0.20 &	0.28&-0.25\\
			(M3)&&   &	&	-0.08&-0.55\\
			(M4)& &	&	&	&-0.60\\
			(M5)&&	 &	 &	 &
		} \overset{Scaling}{\longrightarrow} 
		\kbordermatrix{
			& (M1) & (M2) & (M3) & (M4) & (M5) \\
			(M1) &  &1 &	4 &	4 &	-4 \\
			(M2)&& &	2 &	3&-3\\
			(M3)&&   &	&	1&-6\\
			(M4)& &	&	&	&-6\\
			(M5)&&	 &	 &	 &
		} 
	\end{equation}
}

Afterward, the $\learnWeights$ function in \ahp{} computes the AHP matrix $A$ using the scaled $\Delta$ matrix (i.e., the average correlation gap) as input:

\renewcommand{\kbldelim}{(}
\renewcommand{\kbrdelim}{)}
{\small
	\begin{equation}
		\label{eqn:AHP-mat}
		{\bf A}
		= \kbordermatrix{
			& (M1) & (M2) & (M3) & (M4) & (M5) \\
			(M1)&1&	1 &	4&	4&1/4\\
			(M2)&1&	1&	2&	3&1/3\\
			(M3)&1/4&  1/2&	1&	1&1/6\\
			(M4)& 1/4&	1/3&	1&	1&1/6\\
			(M5)&4&	 3&	 6&	 6&1
		}
	\end{equation}
}

At the end, \ahp{} computes the weighting vector $w$ by solving a minimization problem.
The learned weight vector $w = (w_{M_1}, w_{M_2}, w_{M_3}, w_{M_4}, w_{M_5})$ is $(0.24, 0.24, 0.065, 0.065, 0.39)$, reflecting the importance of each measure for achieving the goal of the user ranking function $g_w$.

Using the AHP interestingness measure $g_w$ from equation \ref{eq:score}, we can rank all patterns in $\mathcal{P}$ provided in Table \ref{tab:patterns}. \ahp{} achieves an overall ranking accuracy of 91\% in this example, correctly ranking the first 5 patterns, the seventh, and the tenth ones.

%
%
%
%
%
%
%
%
%
%
%

\newcommand{\ones}{x^{-1}& (1)\xspace}
\newcommand{\zeros}{x^{-1}& (0)\xspace}
\newcommand{\unbound}{x^{-1}& (*)\xspace}

\newtheorem{defi}{Definition}
\newtheorem{prop}{Proposition}
\newtheorem{corolarry}{Corollary}

\let\comment=\relax
\def\trackingLevel{2}
\newcommand{\comment}[3]{\ifnumcomp{\trackingLevel}{=}{2}{{\color{#1}[\bf\footnotesize{{#2: \textit{#3}}}]}}{}}


\newcommand{\zoo}{{$\tt Zoo$}\xspace}
\newcommand{\Hepatitis}{{$\tt Hepatitis$}\xspace}

\newcommand{\vote}{{$\tt Vote$}\xspace}
\newcommand{\anneal}{{$\tt Anneal$}\xspace}
\newcommand{\chess}{{$\tt Chess$}\xspace}
\newcommand{\mushroom}{{$\tt Mushroom$}\xspace}
\newcommand{\connect}{{$\tt Connect$}\xspace}
\newcommand{\tone}{{$\tt T10$}\xspace}

\newcommand{\tfour}{{$\tt T40$}\xspace}

\newcommand{\pumsb}{{$\tt Pumsb$}\xspace}
\newcommand{\retail}{{$\tt Retail$}\xspace}

\newcommand{\zooI}[2]{{$\tt Zoo\_#1\_#2$}}
\newcommand{\voteI}[2]{{$\tt Vote\_#1\_#2$}}
\newcommand{\annealI}[2]{{$\tt Anneal\_#1\_#2$}}
\newcommand{\chessI}[2]{{$\tt Chess\_#1\_#2$}}
\newcommand{\mushroomI}[2]{{$\tt Mushroom\_#1\_#2$}}
\newcommand{\connectI}[2]{{$\tt Connect\_#1\_#2$}}
\newcommand{\toneI}[2]{{$\tt T10\_#1\_#2$}}
\newcommand{\tfourI}[2]{{$\tt T40\_#1\_#2$}}
\newcommand{\pumsbI}[2]{{$\tt Pumsb\_#1\_#2$}}
\newcommand{\retailI}[2]{{$\tt Retail\_#1\_#2$}}

\newcommand{\rand}{{$\tt rand$}\xspace}
\newcommand{\lex}{{$\tt lex$}\xspace}

\newcommand{\cputimee}{{\tt t}}
\newcommand{\recc}[1]{${\tt R@{#1}}$}
\newcommand{\rec}[1]{${\tt R@{#1\%}}$}

\section{Experiments}
\label{sec:xp}
In this section, we empirically evaluate the effectiveness of our proposed pattern ranking framework \ahp{}. We first introduce our case study on association rules mining and the different oracles we used to simulate user-specific rankings. Next, we present the research questions we aimed to answer, the experimental protocol we followed, and the results we obtained.

\newcommand{\Rule}[2]{{#1}\rightarrow {#2}}

\subsection{Mining associations rules (ARs)} 
\label{sec:case_study}
Our approach, \ahp{}, is experimentally evaluated on association rule mining, which is one of the most important and well-studied tasks in data mining. Association rules are implications of the form $\Rule{X}{Y}$, where $X$ and $Y$ are itemsets such that $X\cap Y=\emptyset$ and $Y \neq \emptyset$. Here, $X$ represents the body of the rule and $Y$ represents its head. The \emph{frequency} of an itemset $X$ in a dataset, denoted by $freq(X)$, is the number of transactions of the dataset containing $X$. The frequency of a rule $\Rule{X}{Y}$ is the frequency of the itemset $X\cup Y$, that is, $freq(\Rule{X}{Y})=freq(X\cup Y)$.

Various interestingness measures for association rules have been proposed, including support, confidence, interest factor, correlation, and entropy. Tan et al. \cite{tan2004selecting} conducted a study on the usefulness of existing measures in different application types and identified seven independent groups of consistent measures having similar properties, as shown in Table \ref{tab:plain}.

\definecolor{Gray}{gray}{0.85}
\newcolumntype{Y}{>{\centering\arraybackslash}X}
\newcolumntype{Z}{>{\raggedleft\arraybackslash}X}
\newcolumntype{a}{>{\columncolor{Gray}}c}

\begin{table}[h]

	\begin{tabularx}{\textwidth}{ccX}
		\hline
		 Groups &$\qquad$ & Measures                                                                  \\ \hline
		1  & &\textbf{Yules Q}, Yules Y, Odds Ratio                                               \\ \hline
		2    &   & \textbf{Cosine}, Jaccard                                                           \\ \hline
		3    &   & \textbf{Laplace}, Support                                                          \\ \hline
		4    &   & \textbf{$\phi$coefficient}, CollectiveStrength, piatetskyShapiro's \\ \hline
		5    &   & \textbf{Goodman Kruskal's}, Gini Index                                             \\ \hline
		6    &   & \textbf{Interest factor},added value, Klosgen K                                    \\ \hline
		7    &   & \textbf{certainty factor}, Mutual Information, Cohen's $\kappa$                           \\ \hline
	\end{tabularx}

	\caption{Independent Subjective Measure Groups}
	\label{tab:plain}
\end{table}

To evaluate our approach, we selected one measure from each of the seven groups of consistent measures identified by Tan et al. \cite{tan2004selecting}. This gave us a set of seven measures that are independent and have similar properties. These measures are highlighted in bold in Table \ref{tab:plain}.

\newcommand{\randemu}{{\sc Rand-Emu}}
\newcommand{\lexemu}{{\sc Lex-Emu}}
\newcommand{\chiemu}{{\sc Chi-Emu}}

\subsection{User feedback emulators}
Since it can be challenging to evaluate an interactive approach with limited user feedback, we simulate user feedback by using three different objective target ranking functions:

\begin{itemize}
\item \randemu: The user-specific ranking is equivalent to a random weighted aggregation function. For each measure $M_i$, we generate a random weight $w_i\in [0,1]$ such that $\sum_{i=1}^{m} w_i=1$.
\item \lexemu: The user-specific ranking follows a lexicographic order on the measures. We define a lexicographic order $lex$ on $\mathcal{M}$ as $lex(\mathcal{M})=\langle l_1, \ldots, l_m \rangle$ such that $\bigcup_{i=1}^{m} l_i = \mathcal{M}$. Given two patterns $(P_1, P_2)$, $P_1$ is preferred to $P_2$ if $(l_i(P_1)>l_i(P_2))$ or $(l_i(P_1)=l_i(P_2) \wedge l_{i+1}(P_1)>l_{i+1}(P_2))$ for $i=1$ to $m-1$.
\item \chiemu: $\chi^2$ is a statistical measure that is a good candidate to emulate user feedback \cite{RINGUEST1986379}. While $\chi^2$ can be a complex function to approximate with non-trivial correlations, it is also a quality measure suggested in \cite{dzyuba2017learning}. We use $\chi^2$ as the target user-specific ranking function over ARs. For an association rule $\Rule{X}{Y}$, the $\chi^2$ value is defined as:
\begin{equation}
\chi^2 (\Rule{X}{Y}) = \frac{ \left(freq(\Rule{X}{Y})- \frac{freq(X) freq (Y)}{N} \right) ^2}{\frac{freq (X) freq (Y)}{N}},
\end{equation}
where $N$ is the number of transactions in the dataset.
\end{itemize}

\newcommand{\RQ}[1]{\textbf{RQ#1}}

\subsection{Research questions}

Our evaluation seeks to address the following research questions:

\begin{itemize}
\item \RQ{1}: \emph{Can we determine a user's preferences for a set of patterns based on a sample of ranked patterns? If so, how much data is needed and how long does the learning process take?}
\item \RQ{2}: \emph{How does our proposed AHP-based learning method compare to the SVM-based baseline for automatically learning user-specific ranking functions?}
\item \RQ{3}: \emph{How effective is \ahp{} in an active learning context? Is the sensitivity heuristic a good choice for query selection?}
\item \RQ{4}: \emph{How effective is \ahp{} in an interactive data mining context?}
\end{itemize}

\subsection{Experimental protocol}

\subsubsection{Implementation settings}
We implemented our \ahp{} approach in Java with two modes: the passive mode, denoted as \ahpp{}, and the active mode, denoted as \ahpa{}. The code is publicly available on GitHub at ${\tt github.com/lirmm/AHPRank}$. We compared our approach to the state-of-the-art SVM-based approach \svm{} \cite{dzyuba2017learning}. We use \svmp{} to denote the passive version and \svma{} to denote the active version, following \cite{dzyuba2013interactive}. All experiments were conducted on an Intel Core $i7$ 2.4 GHz with 16 GB of RAM, with a maximum duration of one hour.

\subsubsection*{Metrics}
We use three metrics to evaluate the performance of our approach:

\begin{enumerate}
\item The Spearman's rank correlation coefficient $\rho$ measures the accuracy of the learned ranking compared to the target ranking over $n$ patterns. It is computed as follows:

\begin{equation}
    \rho= 1-{\frac {6\sum_{i=1}^n (\mathit{rank}_{L}(P_i)-\mathit{rank}_{T}(P_i))^2}{n(n^{2}-1)}}
\end{equation}

\item The recall metric $\text{R}@k$ evaluates the ability of our approach to identify the top $k$ most interesting patterns. It is computed as follows:

\begin{equation}
    \text{R}@k = \frac{\abs{\{\mathit{rank}_{L}(P_i) \leq k: i \in 1..n \wedge \mathit{rank}_{T}(P_i) \leq k\}}}{k}
\end{equation}

\item CPU time in seconds is measured for both the passive and active modes. The waiting time between two queries is also recorded for the active mode.

\end{enumerate}

\subsection{Benchmark Datasets}

We selected several datasets of realistic sizes from the FIMI repository.\footnote{fimi.uantwerpen.be/data/} These datasets have different characteristics that represent various application domains. Table~\ref{tab:datasets} reports the number of transactions (\nbTrans), the number of items (\nbItems), the application domain, and the number of valid rules (\nbRules) corresponding to the initial rules mined using a standard association rules algorithm without any knowledge about the user, for each dataset. The datasets are presented in ascending order of \nbRules.

\begin{table}[tbp]
	
	\centering
	
	\begin{tabularx}{\textwidth}{|Y|Y|Y|Y|c||Y|} \hline
		
		{Dataset} & \multicolumn{1}{c|}{\nbTrans} & \multicolumn{1}{c|}{\nbItems}   & {Density(\%)}& {Type of Data} & \multicolumn{1}{c|}{\nbRules}   \\
		\hline
		\hline
		{\Hepatitis}  & 137 & 68 &50.00&{Disease} & 0.5 M\\
		{\connect}  & 67,557  & 129 &33.33& {Game steps} &1 M\\
		{\mushroom}  & 8,124  & 119 &18.75& {Species of mushrooms} & 1.5 M\\
		{\tfour}  & 100,000  & 1,000 &4.20&{Synthetic dataset} &2 M\\
		{\retail}  & 88,162 & 16,470 &0.06& {Retail market basket data} & 2,5 M\\ \hline
		\multicolumn{6}{r}{ \tfour:{$\tt T40I10D100K$}}
	\end{tabularx}

	\caption{Dataset Characteristics. \label{tab:datasets}} 

\end{table}

\newcommand{\vbm}{${\tt VBM}$}

\newcommand{\timeout}[1]{${\tt TO} $}

\newcommand{\cputime}{${\tt Time}\,$}
\newcommand{\latencytime}{${\tt TQ}\,$}

\subsection{Passive learning results}

\begin{table}[h]
	\centering
	\begin{tabularx}{\textwidth}{|l||Z||Z|Z|Z|Z|Z|Z|Z|} 
		\hline 
		\multirow{2}{*}{\textbf{Datasets}}	&	\multicolumn{8}{|c|}{\textbf{measures}} \\ \cline{2-9}
		&\multicolumn{1}{c|}{(1)}  & \multicolumn{1}{c|}{(2)} 
		& \multicolumn{1}{c|}{(3)} & \multicolumn{1}{c|}{(4)}  
		& \multicolumn{1}{c|}{(5)} &  \multicolumn{1}{c|}{(6)} 
		& \multicolumn{1}{c|}{(7)}& \vbm \\  \hline \hline
		\multicolumn{9}{|c|}{\randemu}\\ \hline	
		{\Hepatitis} &0.78
		&0.92
		&0.29
		&\textbf{0.97}
		&0.79
		&0.85
		&0.45
		&\textbf{0.97}   
		\\ 
		
		{\connect} 
		& 0.68
		& 0.62
		& \textbf{0.95}
		& 0.70
		& 0.60
		& 0.45
		& 0.57   
		&\textbf{0.95}		
		\\ 
		
		{\mushroom} 
		&{0.84}
		&\textbf{0.97}
		&0.28	
		&0.83
		&0.55 
		&0.73	
		&0.36
		&\textbf{0.97}
		\\

		{\tfour}&0.71
		&\textbf{0.99}
		&0
		&0.76
		&\textbf{0.99}
		&\textbf{0.99}
		&0  
		&\textbf{0.99}	
		\\ 
		
		{\retail} &0
		&\textbf{0.98}
		&0.43
		&0.71
		&0.84
		&\textbf{0.98}
		&0.51 
		&\textbf{0.98}		
		\\ \hline

		\multicolumn{9}{|c|}{\lexemu}\\ \hline
		{\Hepatitis} &0.79
		&0.64
		& 0
		&0.68
		&\textbf{0.92}
		&0.63
		&0   
		&\textbf{0.92 }
		\\ 
		
		{\connect} 
		
		&0.21
		&0.47 
		&\textbf{0.76}
		&0.26 
		& 0 
		&0.23
		&0.50 
		&\textbf{0.76}		
		\\ 
		
		{\mushroom} &0.21 
		&\textbf{0.82}
		&0.38	
		&0.77
		&0.78
		&0.37
		&0
		&\textbf{0.82}
		\\

		{\tfour}&0.76
		&0.98
		&0
		&0.77
		&\textbf{0.99}
		&0.97
		&0  
		&\textbf{0.99}	
		\\ 
		
		{\retail} &0.10
		&0.78
		&0.54
		&0.49
		&\textbf{0.84}
		&0.80
		&0.58 
		&\textbf{0.84}		
		\\ \hline 

		\multicolumn{9}{|c|}{\chiemu}\\ \hline
		{\Hepatitis} &0.20
		&	0.92
		&	0.36
		&	\textbf{0.96}
		&0.73
		&0.85
		&0.43  
		&\textbf{0.96}  
		\\ 

		{\connect} 
		&0.09
		&0.15
		& 0 
		&0.02
		& \textbf{0.29}
		& 0.01
		& 0
		&\textbf{0.29}		
		\\ 
		
		{\mushroom} &0.51 
		&0.64
		& 0	
		&0.46	
		&0.28 
		&\textbf{0.88}	
		& 0 
		&\textbf{0.88}
		\\

		{\tfour}&0.71
		&0.98
		&0.17
		&0.64
		&0.98
		&\textbf{0.99}
		&0.29	
		&\textbf{0.99}		
		\\ 
		
		{\retail} &0
		&\textbf{0.94}
		&0.53
		&0.60
		&0.83
		&0.57
		&0.61 
		&\textbf{0.94} \\  \hline

		\multicolumn{9}{|r|}{(1): \yules $\qquad$(2): \cosine $\qquad$ (3): \lap$\qquad$ (4): \lev }\\
		\multicolumn{9}{|r|}{  (5): \lamb $\qquad$ (6): \IFa$\qquad$ (7): \cert  }\\ 
		\hline
	\end{tabularx}
	\caption{Correlation results with user ranking for using the emulators. }
	\label{tab:xp1}
\end{table}
In this section, we address the first two research questions (\RQ{1} and \RQ{2}). To that end, we perform a 5-fold cross-validation on each dataset, where we randomly select $20\%$ of the rules as training data and use the remaining $80\%$ for testing. This type of cross-validation allows us to evaluate the effectiveness of the approaches in learning from relatively small training sets. We report the results averaged over the 5 folds.

\medskip
\noindent
\textbf{A) Analyzing the different user feedback emulators.} 
Table \ref{tab:xp1} presents the results of the correlation analysis between the user-specific ranking functions (\randemu, \lexemu, and \chiemu) and the seven interestingness measures, as well as the \emph{Virtual Best Measure} (\vbm{}), which returns the best rank correlation $\rho$ provided by one of the seven measures. The analysis was performed on all datasets, and the results are reported in terms of Spearman's rank correlation coefficient $\rho$.

The results show that the correlation between a given measure and a user-specific ranking function can vary significantly depending on the dataset. For instance, the \lamb{} measure is highly correlated with \chiemu{} on \tfour{} dataset ($\rho=98\%$), but it is weakly correlated on \connect{} dataset ($\rho=29\%$). Similarly, the correlation between \randemu{} or \lexemu{} and the measures also varies significantly across datasets. However, the \vbm{} approach, which selects the best measure for each dataset, achieves a high level of accuracy (with a mean of $91\%$). These results suggest that a weighted aggregation of the selected measures could lead to a good trade-off between accuracy and robustness.

\begin{table}[tbp]
	\centering	
	
	\begin{tabularx}{\textwidth}{|l||Z|Z|Z|Z||Z|Z|Z|Z|} 
		\hline
		\multicolumn{9}{|c|}{\textbf{\randemu}} \\ \hline
		
		&  \multicolumn{4}{c||}{\svmp} & \multicolumn{4}{c|}{ \ahpp}\\ \hline
		
		&\multicolumn{1}{c|}{$\rho$}&\multicolumn{1}{c|}{\rec{10}}& \multicolumn{1}{c|}{\rec{1}} & \multicolumn{1}{c||}{t(s)}
		&\multicolumn{1}{c|}{$\rho$}&\multicolumn{1}{c|}{\rec{10}}& \multicolumn{1}{c|}{\rec{1}} & \multicolumn{1}{c|}{t(s)}		
		\\ \hline
		{\Hepatitis } 
		&0.94&0.74&0.77&851&0.93&0.79&0.71&9 
		\\ 
		{\connect } 
		&  {-}&-&-& \tt TO &0.95 &0.72&0.63&26 
		\\ 
		{\mushroom }  
		& {-}&-&-& \tt TO&0.89&0.69&0.66&32
		\\ 
		{\tfour  }
		&{-}&-&-& \tt TO&0.99&0.93&0.58 & 63
		\\ 
		{\retail }
		&{-}&-&-& \tt TO&0.93&0.95&0.96&66
		\\ \hline \hline
		{{$mean$}} & - & - & -& -& 0.94&0.82&0.71& 39\\ \hline 		\hline
		
		\multicolumn{9}{|c|}{\textbf{\lexemu}} \\ \hline

		&  \multicolumn{4}{c||}{\svmp} & \multicolumn{4}{c||}{ \ahpp}\\ \hline
		
		&\multicolumn{1}{c|}{$\rho$}&\multicolumn{1}{c|}{\rec{10}}& \multicolumn{1}{c|}{\rec{1}} & \multicolumn{1}{c||}{t(s)}
		&\multicolumn{1}{c|}{$\rho$}&\multicolumn{1}{c|}{\rec{10}}& \multicolumn{1}{c|}{\rec{1}} & \multicolumn{1}{c|}{t(s)}		
		\\  \hline
		{\Hepatitis  } 
		&0.99&0.99&0.99&902&0.92&0.86&0.81& 13\\ 
		{\connect  } 
		&  {-}&-&-& \tt TO &0.60 &0.46&0.50&21\\ 
		{\mushroom  }  
		& {-}&-&-& \tt TO&0.86&0.56&0.65&32\\ 
		{\tfour  }
		&{-}&-&-& \tt TO&0.99&0.93&0.46 & 71\\ 
		{\retail  }
		&{-}&-&-& \tt TO&0.78&0.68&0.65&68\\ \hline \hline
		{{$mean$}} & -&-&-& - & 0.83 &0.70&0.61 & 41\\ \hline 		\hline
		
		\multicolumn{9}{|c|}{\textbf{\chiemu}} \\ \hline
		
		&  \multicolumn{4}{c||}{\svmp} & \multicolumn{4}{c|}{ \ahpp}\\ \hline
		
		&\multicolumn{1}{c|}{$\rho$}&\multicolumn{1}{c|}{\rec{10}}& \multicolumn{1}{c|}{\rec{1}} & \multicolumn{1}{c||}{t(s)}
		&\multicolumn{1}{c|}{$\rho$}&\multicolumn{1}{c|}{\rec{10}}& \multicolumn{1}{c|}{\rec{1}} & \multicolumn{1}{c|}{t(s)}		
		\\ \hline
		{\Hepatitis  } 
		&0.99& 0.97& 0.92 & 979 &{0.94}&{0.90}& 0.77 & 10\\ 
		{\connect  } 
		& {-}&-&-& \tt TO & 0.15 &{0.28}&{0.67}& 17\\ 
		{\mushroom  }  
		& {-}&-&-& \tt TO & 0.98&0.91&0.93 & 28\\ 
		{\tfour  }
		& {-}&-&-& \tt TO&0.99&0.99&0.96& 84\\ 
		{\retail  }
		& {-}&-&-& \tt TO& 0.91&0.96&0.95 & 50\\ \hline \hline
		{{$mean$}} &-&-&-& - & 0.79 &0.81&0.86 & 38\\ \hline 		\hline
		
	\end{tabularx}
	\caption{5-fold cross-validation results (passive learning).}
	\label{tab:xp2}
\end{table}

\medskip
\noindent
\textbf{B) Comparing \ahpp{} with \svmp{}.} 
Table \ref{tab:xp2} presents the results of the k-folds cross-validation for \svmp{} and \ahpp{}, where we report the averaged rank correlation $\rho$, recall values (\rec{10} and \rec{1}), and CPU time in seconds averaged over the folds.
Since \svmp{} can handle a training set of up to $100K$ rules within an hour, we compared it with our approach only on the \Hepatitis{} dataset.
We observed that \svmp{} outperformed \ahpp{} in terms of ranking accuracy for all user feedback emulators, although \ahpp{} remained competitive with acceptable accuracy.
The same trend was observed in terms of recall at the $10\%$ and $1\%$ top of the ranking (\rec{10} and \rec{1}).
We noted that \ahpp{} achieved a high correlation with the user ranking functions on most datasets.
However, on \connect{}, we observed a weak correlation with \chiemu{}, while a high correlation was observed with \randemu{} and \lexemu{}.
This can be attributed to the fact that \randemu{} and \lexemu{} are linear functions expressed with the given $7$ interestingness measures, which explains their high accuracy, while \chiemu{} is a complex function that requires more extensive statistics to learn.
Notably, the high accuracy of \svmp{} came at the expense of longer running times. For instance, it took more than $15$ minutes for \svmp{} to learn from a training set of $100K$ rules, and exceeding $100K$ rules, \svmp{} required more than one hour, while \ahpp{} was able to handle $7.5M$ rules in less than $4$ minutes.\\
In what follows, the observations and the conclusions drawn from \chiemu{} remain true for \randemu{} and \lexemu.
For the sake of simplicity, we only report the results on the complex function \chiemu{}.

\begin{figure}[H]
	\centering
	\includegraphics[width=\textwidth]{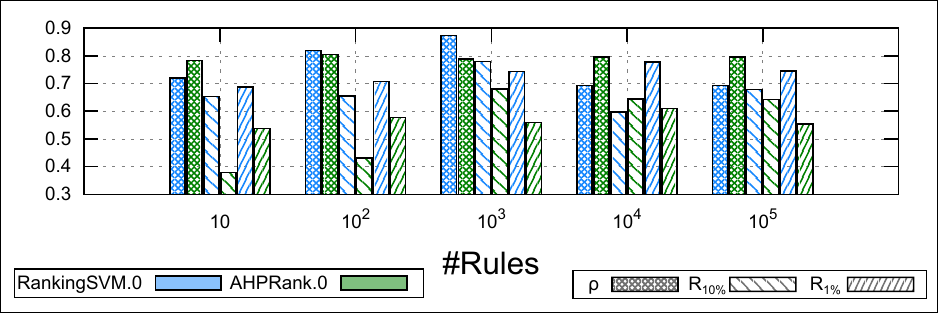} %
	
	\caption{ Learning accuracy comparison between \svmp{} and \ahpp{}  learning \chiemu{} on different training data size.}%
	\label{fig:passive1}%
\end{figure}

\medskip
\noindent
\textbf{C) Impact of varying the size of the training data on the learning.} 
 It is crucial to emphasize that our proposed \ahp{} approach aims to achieve a fast learning process while providing high accuracy in ranking patterns.
To further support our observation on the scalability of \ahpp{} compared to \svmp{}, we present in Figures \ref{fig:passive1} and \ref{fig:passive2} a performance comparison by varying the size of the training data when learning the \chiemu{} function, as similar results were obtained for \randemu{} and \lexemu{} functions.
For each dataset, we randomly select $nb$ rules, and we compare the two approaches for $nb\in{10, 100, 1K, 10K, 100K}$. The results are averaged over ten runs.

Regarding the Spearman correlation $\rho$, Figure \ref{fig:passive1} shows that the two approaches have a discrepancy of only $5\%$ when the training data does not exceed $1K$. However, the gap becomes significantly more substantial (exceeding $10\%$) and in favor of \ahpp{} when the training data contains $10K$ and $100K$ rules.

In terms of recall at $10\%$ and $1\%$, \svmp{} outperforms \ahpp{} with a gap of $23\%$ and $15\%$ at \rec{10} and \rec{1}, respectively, when the training data contains only $10$ rules. However, as the size of the training data increases, the gap between the two approaches becomes narrower, with a difference of less than $5\%$ for \rec{10}, and remaining relatively constant for \rec{1}.

\begin{figure}[H]
	\centering
	
	\includegraphics[width=\textwidth]{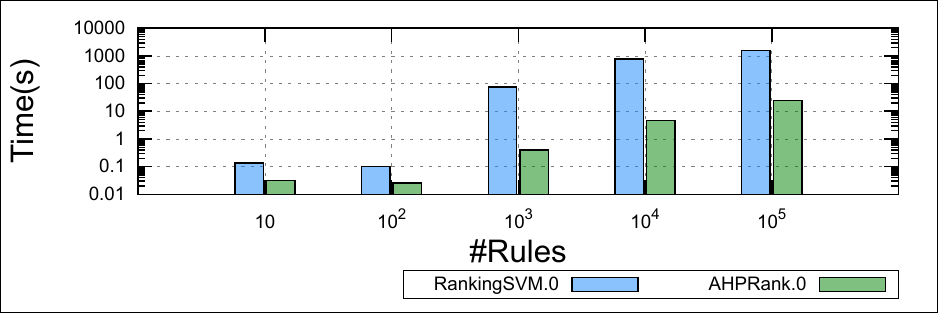} %
	\caption{ CPU time comparison between \svmp{} and  \ahpp{}  learning using \chiemu{} emulator on different training data sizes averaged overall datasets.}%
	\label{fig:passive2}%
\end{figure}

In terms of CPU time, Figure \ref{fig:passive2} demonstrates that \svmp{} can process training data of up to $100$ rules in less than a second, whereas it takes more than one minute for $1K$ rules, $13 minutes$ for $10K$ rules, and over $26 minutes$ for $100K$ rules.
In contrast, \ahpp{} is capable of handling training data ranging from $100$ to $100K$ rules in a time span ranging from $0.03$ to $24.86$ seconds, while still maintaining a ranking accuracy comparable to that of \svm.

\subsection{Active learning results}

In this section, we aim to address research questions \RQ{2} and \RQ{3}, which involve comparing the active learning versions of \svm{} and \ahp{}. To achieve this, we employ a straightforward process that involves asking the user ranking queries on pairs of patterns (e.g., \emph{do you prefer $P_i$ to $P_j$?}).

\medskip
\noindent
\textbf{A) Evaluating the effectiveness of our sensitivity-based heuristic.} 
In this section, we aim to evaluate the effectiveness of our Sensitivity-Based Generator (SBG) (see section ~\ref{sec-AHP-based-learning}) in improving the active learning process of \ahpa{}. To do so, we compare SBG to a Random Generator (RG), where RG selects a pair of patterns randomly from \nbRules{} and submits them to both \svma{} and \ahpa{}. We conduct our experiments with human-in-the-loop, setting the number of queries (i.e., iterations of Algorithm \ref{alg:ActiveLearning}) to a maximum of $20$ queries. We repeat the experiment $10$ times and take the average result to account for the sampling step of our SBG and the randomness of the RG approach. After a few tests, we set the sample size $\mathcal{X}$ picked at line ~\ref{ahp:sample} of algorithm ~\ref{alg:ActiveLearning} to $\theta=10^3$, providing a good trade-off between time selection and the accuracy of the selected pair.

Figure \ref{fig:active} presents a scatter plot of $20$ iterations comparing the performance of \svma{} and \ahpa{} with both RG and SBG. We report the Spearman correlation $\rho$, recall at $10\%$, and recall at $1\%$ for each iteration. The results show that the use of SBG outperforms RG, regardless of the learning algorithm, as can be observed from the scatter plot.

\begin{figure}
	\centering
	\includegraphics[scale=0.6]{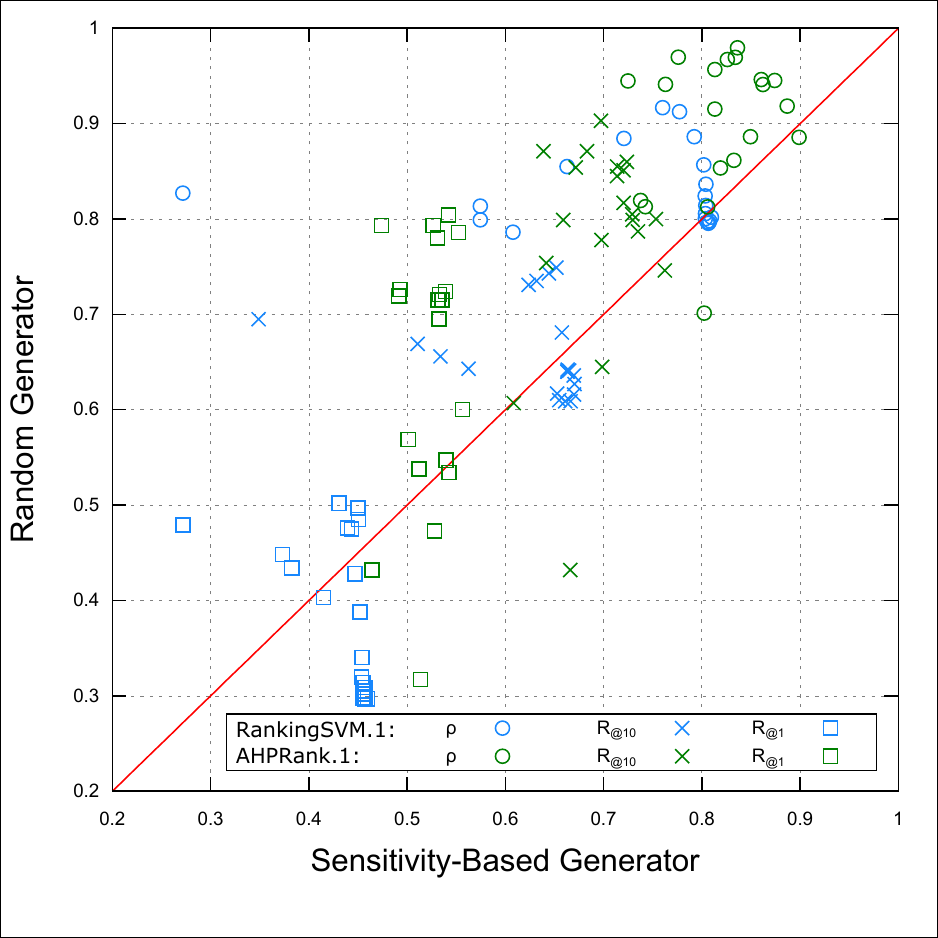} %
	
	\caption{Qualitative comparison of Random vs. Sensitivity-Based Query Generator using the Spearman correlation $\rho$  and the recall at $10\%$ and $1\%$ of each iteration.}%
	\label{fig:active}%
\end{figure}

\newcommand{\ci}{${\tt CI}$}

To support our observation on the effectiveness of SBG, we conducted a statistical test using the Wilcoxon Signed-Rank Test. We used a one-tailed alternative hypothesis with the null hypothesis that RG is more efficient than SBG: H0: The accuracy using RG $\geq$ The accuracy using SBG. The alternative hypothesis H1 states that SBG outperforms RG. With this statistical test, we concluded that the use of SBG is more efficient than RG (H1 accepted).

Table \ref{tab:wilcoxon} reports the p-value, z-score, and confidence interval (CI) of each test. Except for the case of (\svma,\rec{10}), we had strong evidence to reject the null hypothesis based on the p-value column. The CI column clearly shows that the use of SBG is better than RG. Consequently, we will use the SBG heuristic in \svma{} and \ahpa in the following.

\begin{table}[h]
	\centering
	
	\begin{tabularx}{\textwidth}{|Y|Y|YYY|} 
\hline
\textbf{Approaches} & \textbf{Metrics} & $p$-$value$ & $z$-$score$ & \ci	\\ 
\hline		
\multirow{3}{*}{\svma} & $\rho$ & $.00226$& $-2.8373$&  $99\%$\\ 
 & \rec{10} & $.12302$ & $-1.1573$ & $70\%$ \\  
  & \rec{1} & $.06301$ & $-1.5306$& $90\%$ \\ \hline
  \multirow{3}{*}{\ahpa} & $\rho$ & $.00034$& $-3.3973$& $99.9\%$ \\ 
  & \rec{10} & $.00205$ & $-2.8746$& $99\%$\\ 
  & \rec{1} &$.00139$ & $-2.9866$ & $99\%$\\ \hline
	\end{tabularx}

	\caption{ Wilcoxon Signed-Rank Test (RG vs. SBG).} 		\label{tab:wilcoxon}
\end{table}

\begin{table}[tbp]
	\centering
	
	\begin{tabularx}{\textwidth}{|r|r|XXXXX|} 
		\hline
		\multicolumn{2}{|c|}{}  & (a)& (b)& (c)& (d)& (e)      \\ 
		\hline \hline
		\multicolumn{7}{|c|}{ \textbf{10 Queries} } \\ \hline
		
		\multirow{4}{*}{\rotatebox[origin=c]{0}{(1)} }  
		
		& $\rho$  & 0.73& 0.04& 0.69& {0.99}& 0.97\\ 
		& \rec{10} & 0.61& 0.28& 0.38&{0.92}& {0.91} \\ 
		& \rec{1}  & 0.44& 0.49& {0.36}& 0.45& 0.67\\ 
		& \cputime     & {0}        & {0}       & {0}         & {0}         & {0}              \\ 
		\hline
		
		\multirow{4}{*}{\rotatebox[origin=c]{0}{(2)} }  
		
		& $\rho$  &\textbf{0.78}& \textbf{0.10}& \textbf{0.93}& {0.99}& \textbf{0.99}\\ 
		& \rec{10} & \textbf{0.66}& \textbf{0.36}& \textbf{0.80}&{0.97}& {0.91} \\ 
		& \rec{1}  & \textbf{0.51}& \textbf{0.50}& \textbf{0.71}& \textbf{0.81}& \textbf{0.93}\\ 
		\hline

		\hline

		\multicolumn{7}{|c|}{ \textbf{50 Queries}} \\ \hline
		
		\multirow{4}{*}{\rotatebox[origin=c]{0}{(1)} }  
		
		& $\rho$  & 0.94& \textbf{0.20}& 0.71& {0.99}& {0.98}\\ 
		& \rec{10} & 0.77& 0.10& 0.39& 0.92& {0.91} \\ 
		& \rec{1}  & 0.57& 0& {0.36}& 0.45& 0.67\\ 
		& \cputime   & \textbf{0} & \textbf{0} & \textbf{0} & 2 & 3\\
		\hline
		
		\multirow{4}{*}{\rotatebox[origin=c]{0}{(2)} }  
		& $\rho$  & \textbf{0.95}& 0.10 & \textbf{0.93}& {0.99}& \textbf{0.99}\\ 
		& \rec{10} & \textbf{0.83}& \textbf{0.40}& \textbf{0.72}& \textbf{0.94}& {0.91} \\ 
		& \rec{1}  & \textbf{0.66}& \textbf{0.45}& {0.51}& \textbf{0.64}& \textbf{0.93}\\ 
		\hline

		\multicolumn{7}{|c|}{ \textbf{100 Queries}} \\ \hline
		
		\multirow{4}{*}{\rotatebox[origin=c]{0}{(1)} }  
		
		& $\rho$  & 0.94& \textbf{0.29}& 0.72& {0.99}& 0.98\\ 
		& \rec{10} & 0.75& 0.20& 0.41& 0.92& {0.91}\\ 
		& \rec{1}  & 0.53 & 0 & {0.36} & 0.45& 0.67\\ 
		& \cputime   & 0 & 2 & 3& 3& 10\\		\hline
		
		\multirow{4}{*}{\rotatebox[origin=c]{0}{(2)} }  
		& $\rho$  & \textbf{0.98}& 0.10& \textbf{0.81}& {0.99}& \textbf{0.99}\\ 
		& \rec{10} & \textbf{0.89}& \textbf{0.35}& \textbf{0.49}& \textbf{0.95}&{0.91}\\ 
		& \rec{1}  & \textbf{0.78} & \textbf{0.42}& \textbf{0.37} &\textbf{0.74}& \textbf{0.93}\\ 
		\hline
		\hline
		
		\multicolumn{7}{|r|}{(1): \svma $\qquad$ (2): \ahpa}\\ 
		
				\multicolumn{7}{|r|}{(a): \Hepatitis$\quad$ (b): \connect$\quad$ (c): \mushroom$\quad$ (d): \tfour$\quad$ (e):\retail}\\ \hline
		
	\end{tabularx}
	\caption{Qualitative evaluation of \svma{} vs \ahpa. \label{tab:active}}
\end{table}

\medskip
\noindent
\textbf{B) Comparing \ahpa{} with \svma{}.} 
Table~\ref{tab:active} presents a comparison between \svma{} and \ahpa{} on five datasets, reporting the Spearman correlation $\rho$, recall at $10\%$ and $1\%$ stopping criterion $T$ at 10, 50, and 100 queries, and the average latency time \cputime{} between two queries for \svma{} (never exceeding the latency bound of 0.1 seconds). The findings hold for the \randemu{}, \lexemu{}, and \chiemu{} functions.

The main observation from Table~\ref{tab:active} is that \ahpa{} outperforms \svma{}. Looking at the \Hepatitis{} dataset, for example, \ahpa{} achieves an accuracy of $78\%$ with only 10 queries, while \svma{} achieves $73\%$. At 50 queries, \ahpa{} achieves an accuracy of $95\%$, compared to $94\%$ for \svma{}. At 100 queries, \ahpa{} reaches $98\%$ accuracy, while \svma{} remains stable at $94\%$. In terms of recall, \ahpa{} discovers the most relevant patterns in the first 50K and 5K patterns (out of 500K) with accuracies of $78\%$ and $66\%$ over 10 queries, respectively, compared to $61\%$ and $44\%$ for \svma{}.

For the \retail{} dataset, \ahpa{} achieves high accuracy over 10 queries, which remains stable over the following 90 queries. \svma{} results are less impressive on \retail{} and particularly on the recall metric. However, the main difference is in the waiting time \cputime, where \svma{} can keep the user waiting for more than 10 seconds between two queries. The same observation can be made on the other datasets, where \svma{} can be hampered by overall waiting time, even with queries of size 2. In contrast, \ahpa{} shows an instantaneous behavior, taking less than 0.1 seconds between two queries. This represents a limitation in the use of \svma{}, especially when the learning is integrated into an interactive data mining process, where a reasonable latency time for a human user is around a few seconds \cite{Lallemand12}.

\newcommand{\err}{${\tt Err}$}
\subsection{Interactive learning results}
In this section, we address the last research question \RQ{4}, and conduct two experiments to evaluate the robustness and performance of our approach in the presence of human errors.

Our first experiment aims to evaluate the robustness of our approach in the face of human errors. To simulate situations where the user feedback may be incorrect, we randomly select a set of queries and swap the user preference with a probability \err{}. Specifically, a user who prefers pattern $P_i$ over $P_j$ will mistakenly prefer $P_j$ over $P_i$ with a probability \err{}.

Figure \ref{fig:active-err} compares the performance of \svma{} and \ahpa{} under different levels of human error, ranging from $0\%$ to $40\%$, and reports results averaged over $10$ runs on the whole set of datasets using \chiemu{} and submitting $20$ queries to the user. Our results show that \ahpa{} is quite stable and robust even with high levels of human error, up to $40\%$ ($8$ mistakes out of $20$). The overall correlation between the learned function and the user one ($\rho$) remains stable at $77\%$ without errors and drops only to $70\%$ under \err$=40\%$. However, \svma{} is stable only up to \err$=20\%$, and then its accuracy drops significantly, with $\rho$ decreasing from $71\%$ to $54\%$. In terms of recall, the decline is less than $2\%$ at \rec{10} and $8\%$ at \rec{1} under \ahpa{}, while \svma{} experiences a decline exceeding $10\%$ at both \rec{10} and \rec{1}. In terms of CPU time, we observe waiting times ranging from $6$ to $12$ seconds under \svma{}, whereas our approach never exceeds $0.02$ seconds between two queries.

These results suggest that our approach is robust and can handle user errors, while \svma{} is less robust and can suffer from significant drops in accuracy when faced with human errors. Additionally, our approach provides faster query processing times, which is critical in interactive data mining scenarios where human users expect near-instantaneous feedback.

\begin{figure}[h]
	\centering
	\includegraphics[width=\textwidth]{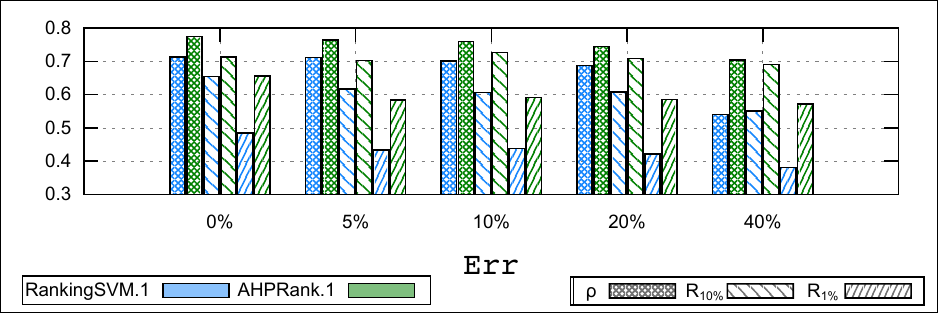} %
	
	\caption{ Learning under mistakes ($20$ queries).}%
	\label{fig:active-err}%
\end{figure}

\newcommand{\lexrand}{{\sc Lex2Rand}}

Our second experiment aims to evaluate the robustness of our approach when faced with an undecided user who begins with a set of preferences $A$ but ends up with preferences $B$. Here, the initial user preferences $A$ are biased towards the presented patterns until the user is more comfortable with preferences $B$.

To simulate such situations where the user's target function changes during the learning process, we conduct an experiment where we start with \lexemu{} as the target function to learn, and after $x$ queries, we switch to \randemu{}. This learning process takes $(20-x)$ more queries, and we refer to it as the \lexrand{} target. We choose to switch from \lexemu{} to \randemu{} because the two functions are linear. We compare \svma{} and \ahpa{} on \lexrand{}($x$) target for a total of $20$ queries, where $x\in{0,5,10}$. Note that with $x=0$, \lexrand{(0)} is equivalent to \randemu{}.

Figure \ref{fig:active-swap} shows the averaged results of $10$ runs on all datasets over $20$ queries. The main observation that we can draw is that \ahpa{} is stable even when faced with a changing linear function to learn, even if the user's preferences change halfway through the learning process (i.e., after $10$ queries). However, changing the target function during the learning process can significantly impact the accuracy of \svma{} (a decline of $44\%$ in terms of $\rho$).
Regarding CPU time, our approach is $50$ times faster than \svma{}.

\begin{figure}[h]
	\centering
	\includegraphics[width=\textwidth]{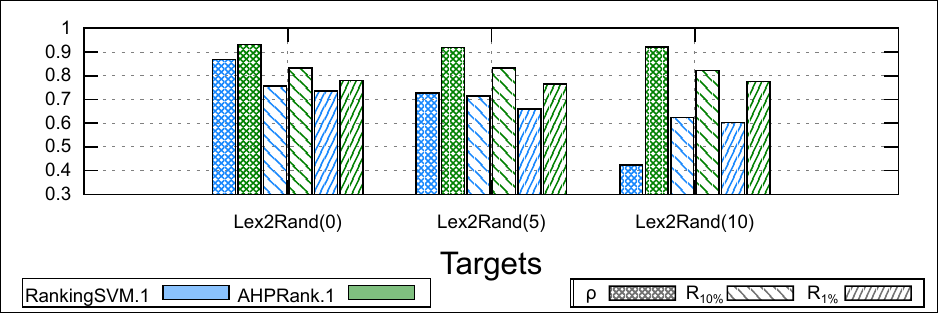} %
	
	\caption{ Learning under swaps ($20$ queries)}%
	\label{fig:active-swap}%
\end{figure}

To summarize, our research demonstrates that \ahp{} provides a favorable balance between ranking accuracy and running time. With a time complexity linear to the size of the training data, it achieves a high ranking accuracy and takes only a few seconds to run, making it a practical solution for industrial use.

\section{Conclusion}

In this paper, we presented a novel framework for learning pattern ranking functions using the AHP multicriteria decision making method. Our proposed \ahp{} algorithm can be applied in both passive and active learning modes, allowing the user to rank subsets of patterns according to their preferences. The learned weights are used to aggregate all the measures into a single ranking function using a weighted linear formulation, which was shown experimentally and statistically to be as close as possible to the user's ranking.
We applied our framework to the case study of Association Rules Mining and compared the results with state-of-the-art learning approaches. Our experiments showed that the \ahp{} algorithm can efficiently learn the ranking function and outperform existing methods in terms of both ranking accuracy and running time.
Overall, our proposed framework offers a practical and efficient approach for learning pattern ranking functions, with potential applications in a wide range of domains.

\bmhead{Acknowledgements}
Funded by the European Union through the AI4CCAM project (Trustworthy AI for Connected, Cooperative Automated Mobility) under grant agreement No 10107691. Views and opinions expressed are however those of the author(s) only and do not necessarily reflect those of the European Union or CINEA. Neither the European Union nor the granting authority can be held responsible for them.


\bibliography{reference}



\end{document}